\documentclass[10pt,journal,doublecolumn,compsoc]{IEEEtran}

\ifCLASSOPTIONcompsoc
  \usepackage[nocompress]{cite}
\else
  \usepackage{cite}
\fi
\ifCLASSINFOpdf
\else
\fi
\usepackage{graphicx,amsmath,amsfonts,amssymb,algorithmic,algorithm,comment,color}

\def\ff{\frac}

\def\ii{^{-1}}
\def\ll{\lambda}

\def\arr{\rightarrow}

\def\cE{\mathcal{E}}
\def\cA{\mathcal{A}}
\def\EE{\mathbb{E}}

\def\arr{\rightarrow}
\def\bal{\begin{align}}

\newtheorem{theorem}{\textbf{Theorem}}
\newtheorem{proof}{\textbf{Proof}}
\newtheorem{lemma}{\textbf{Lemma}}

\newtheorem{definition}{\textbf{Definition}}

\newtheorem{remark}{\textbf{Remark}}

\begin{document}
\title{Maximum Likelihood Latent Space Embedding of Logistic Random Dot Product Graphs}

\author{Luke O'Connor, Muriel M{\'e}dard and Soheil Feizi 
\IEEEcompsocitemizethanks{\IEEEcompsocthanksitem S. Feizi is the corresponding author. L. O'Connor and S. Feizi contributed equally to this work. L. O'Connor is with Harvard University. S. Feizi is with Stanford Universiy. M. M\'edard is with Massachusetts Institute of Technology (MIT). \protect}
}

\IEEEtitleabstractindextext{%
\begin{abstract}
A latent space model for a family of random graphs assigns real-valued vectors to nodes of the graph such that edge probabilities are determined by latent positions. Latent space models provide a natural statistical framework for graph visualizing and clustering. A latent space model of particular interest is the Random Dot Product Graph (RDPG), which can be fit using an efficient spectral method; however, this method is based on a heuristic that can fail, even in simple cases. Here, we consider a closely related latent space model, the Logistic RDPG, which uses a logistic link function to map from latent positions to edge likelihoods. Over this model, we show that asymptotically exact maximum likelihood inference of latent position vectors can be achieved using an efficient spectral method. Our method involves computing top eigenvectors of a normalized adjacency matrix and scaling eigenvectors using a regression step. The novel regression scaling step is an essential part of the proposed method. In simulations, we show that our proposed method is more accurate and more robust than common practices. We also show the effectiveness of our approach over standard real networks of the karate club and political blogs.
\end{abstract}

}

\maketitle

\IEEEdisplaynontitleabstractindextext
\IEEEpeerreviewmaketitle
\ifCLASSOPTIONcompsoc
\IEEEraisesectionheading{\section{Introduction}}
\else
\section{Introduction}
\fi

\IEEEPARstart{C}{lustering} over graphs is a classical problem with applications in systems biology, social sciences, and other fields \cite{girvan, butenko, mishra, lee}. Although most formulations of the clustering problem are NP-hard \cite{schaef}, several approaches have yielded useful approximate algorithms. The most well-studied approach is spectral clustering. Most spectral methods are not based on a particular generative network model; alternative, model-based approaches have also been proposed, using loopy belief propagation \cite{decelle}, variational Bayes \cite{airoldi}, Gibbs sampling \cite{snijders}, and semidefinite programming \cite{hajek,Amini}.

Many spectral clustering methods are derived by proposing a discrete optimization problem, and relaxing it to obtain a continuous, convex optimization whose solution is given by the eigenvectors of a normalized adjacency matrix or Laplacian. A post-processing step, typically $k$-means, is used to extract clusters from these eigenvectors \cite{ng,luxburg}. Different matrix normalizations/transformations include Modularity \cite{girvan}, Laplacian \cite{mohar}, normalized Laplacian \cite{chung,shi,rohe}, and Bethe Hessian \cite{saade}. These methods are often characterized theoretically in the context of the Stochastic Block Model (SBM) \cite{holland}, a simple and canonical model of community structure. Theoretical bounds on the detectability of network structure for large networks have been established for stochastic block models \cite{decelle,mossel1,mossel2,massoulie}. Several spectral methods have been shown to achieve this recovery threshold \cite{saade,krzakala,newman1}. Strong theoretical and empirical results have also been obtained using SDP-based \cite{Amini,hajek} and Belief Propagation-based \cite{decelle} methods, which often have higher computational complexity than spectral methods. A threshold has also been discovered for perfect clustering, i.e. when community structure can be recovered with zero errors \cite{abbe, hajek}.

An alternative approach to the clustering problem is to invoke a latent space model. Each node is assigned a latent position $v_i$, and the edges of the graph are drawn independently with probability $p_{ij}=g(v_i,v_j)$ for some $g(\cdot)$. Hoff et al. \cite{hoff} considered two latent space models, the first of which was a distance model,
$$P_{i,j}=l(-||v_i-v_j||-\mu+\beta x_{ij}),\quad l(x)=1/(1+\exp(-x))$$
where edge probabilities depend on the Euclidean distance between two nodes. $x_{ij}$ is a fixed covariate term, which is not learned. Their second model is called a projection model:
\bal\label{eq:hoff}
P_{i,j}=l(\ff{v_i\cdot v_j}{||v_j||}+\beta x_{ij}-\mu).
\end{align}
Hoff et al. suggest to perform inference using an MCMC approach over both models. Focusing on the distance model, they have also extended their approach to allow the latent positions to themselves be drawn from a mixture distribution containing clusters \cite{shortreed,handcock}. Efforts to improve the computational efficiency have been made in references \cite{townshend,friel}, and with a related mixed membership blockmodel in reference \cite{airoldi}.

Young et al. \cite{young} introduced the Random Dot Product Graph (RDPG), in which the probability of observing an edge between two nodes is the dot product between their respective latent position vectors. The RDPG model can be written as
$$P_{i,j}=v_i\cdot v_j.
$$
This model is related to the projection model of Hoff et al. \eqref{eq:hoff}. The RDPG provides a useful perspective on spectral clustering, in two ways. First, it has led to theoretical advances, including a central limit theorem for eigenvectors of an adjacency matrix \cite{athreya} and a justification for $k$-means clustering as a post processing step for spectral clustering \cite{sussman}. Second, as a natural extension of the SBM, the RDPG describes more comprehensively the types of network structures, such as variable degrees and mixed community memberships, that can be inferred using spectral methods.

Let $V$ be the $n\times d$ matrix of latent positions, where $n$ is the number of nodes and $d$ is the dimension of the latent vectors ($d\leq n$). Sussman et al. \cite{sussman} proposed an inference method over the RDPG based on the heuristic that the first eigenvectors of $A$ will approximate the singular vectors of $V$, as $E(A)=VV^T$. They also characterized the distribution of the eigenvectors of $A$ given $V$. This heuristic can fail, however, as the first eigenvector of the adjacency matrix often separates high-degree nodes from low-degree nodes, rather than separating the communities. This problem occurs even in the simplest clustering setup: a symmetric SBM with two clusters of equal size and density. Therefore, we were motivated to develop a more robust inference approach.

In this paper, we consider a closely related latent space model, the {\it Logistic RDPG}, which uses a logistic link function mapping from latent positions to edge probabilities. Like the previously-studied RDPG, the logistic RDPG includes most SBMs as well as other types of network structure, including a variant of the degree corrected SBM. The logistic RDPG is also similar to the projection model, which uses a logistic link function but models a directed graph (with $p_{i,j}\neq p_{j,i}$). Over this model, we show that the maximum likelihood latent-position inference problem admits an asymptotically exact spectral solution. Our method is to take the top eigenvectors of the mean-centered adjacency matrix and to scale them using a logistic regression step. This result is possible because over the logistic model specifically, the likelihood function separates into a linear term that depends on the observed network and a nonlinear penalty term that does not. Because of its simplicity, the penalty term admits a Frobenius norm approximation, leading to our spectral algorithm. A similar approximation is not obtained using other link functions besides the logistic link.

We show that the likelihood of the approximate solution approaches the maximum of the likelihood function when the graph is large and the latent-position magnitudes go to zero. The asymptotic regime is not overly restrictive, as it encompasses many large SBMs at or above the detectability threshold \cite{decelle}. We compare the performance of our method in the graph clustering problem with spectral methods including the Modularity method \cite{girvan}, the Normalized Laplacian \cite{mohar} and the Bethe Hessian \cite{saade}, and the SDP-based methods \cite{Amini, hajek}. We show that our method outperforms these methods over a broad range of clustering models. We also show the effectiveness of our approach over real networks of karate club and political blogs.

\section{Logistic Random Dot Product Graphs}\label{sec:optimization}

In this section, we introduce the logistic RDPG, describe our inference method, and show that it is asymptotically equivalent to the maximum likelihood inference.

\subsection{Definitions}
Let  $\cA$ be the set of adjacency matrices corresponding to undirected, unweighted graphs of size $n$.

\begin{definition}[Stochastic Block Model]\label{def:sbm}
\textup{The Stochastic Block Model is a family of distributions on $\cA$ parameterized by $(k,c,Q)$, where $k$ is the number of communities,  $c\in [k]^n$ is the community membership vector and $Q\in[0,1]^{k\times k}$ is the matrix of community relationships. For each pair of nodes, an edge is drawn independently with probability
$$P_{i,j}:=Pr(A_{ij}=1|c_i,c_j,Q)=Q_{c_i,c_j}.$$}
\end{definition}

Another network model that characterizes low dimensional structures is the {\it Random Dot Product Graph} (RDPG). This class of models includes many SBMs.

\begin{definition}[Random Dot Product Graph]\label{def:RDPG}
\textup{The Random Dot Product Graph with link function $g(\cdot)$ is a family of distributions on $\cA$ parameterized by an $n\times d$ matrix of latent positions $V\in \mathbb{R}^{n\times d}$. For each pair of nodes, an edge is drawn independently with probability
\begin{align}
P_{i,j}:=Pr(A_{ij}=1|V)=g(v_i\cdot v_j),
\end{align}
where $v_i$, the $i$-th row of the matrix $V$, is the latent position vector assigned to node $i$.}
\end{definition}

The RDPG has been formulated using a general link function $g(\cdot)$ \cite{young}. The linear RDPG, using the identity link, has been analyzed in the literature because it leads to a spectral inference method. We will refer to this model as either the linear RDPG or as simply the RDPG. In this paper, we consider the Logistic RDPG:

\begin{definition}[Logistic RDPG]
\textup{The logistic RDPG is the RDPG with link function:
\begin{align}
g(x)=l(x-\mu),\quad l(x):=\ff{1}{1+e^{-x}},
\end{align}
where $\mu$ is the offset parameter of the logistic link function.}
\end{definition}

Note that this model is similar to the projection model of Hoff et al. \cite{hoff} \eqref{eq:hoff}. The projection model is for a directed graph, with $P_{i,j}\neq P_{j,i}$ owing to the division by $||v_j||$.

\begin{remark}\textup{
The parameter $\mu$ in the logistic RDPG controls the sparsity of the network. If the latent position vector lengths are small, the density of the graph is}
\bal
 \ff{1}{n(n-1)}\sum_{i<j}\EE(A_{ij})\approx l(-\mu).\end{align}
\end{remark}
A logistic RDPG with this property is called {\it centered}. In Section \ref{subsec:inference-method}, we show that asymptotically exact maximum likelihood inference of latent positions over the centered logistic RDPG can be performed using an efficient spectral algorithm.

For a general RDPG, the ML inference problem is:
\begin{definition}[ML inference problem for the RDPG]\label{def:ML-RDPG}
\textup{Let $A\in \cA$. The ML inference problem over the RDPG is:
\begin{align}\label{opt:ML-RDPG}
\max_X  \quad &\sum_{i\neq j} A_{ij}\log g(X_{i,j}) + (1-A_{ij})\log(1-g(X_{i,j})),\\
&X=VV^T,\quad V\in \mathbb{R}^{n\times d}.\nonumber
\end{align}}
\end{definition}
Note that in the undirected case, this objective function is twice the log likelihood.
\begin{remark}\textup{
A convex semidefinite relaxation of Optimization \eqref{opt:ML-RDPG} can be obtained for some link functions as follows:
\begin{align}\label{opt:sdp-relax1}
\max_X  \quad &\sum_{i\neq j} A_{ij}\log g(X_{ij}) + (1-A_{ij})\log(1-g(X_{ij})),\\
&X\succeq 0,\nonumber
\end{align}
where $X\succeq 0$ means that $X$ is psd. For example, this optimization is convex for the logistic link function (i.e., $g(x)=l(x-\mu)$) and for the linear link function (i.e., $g(x)=x$). However, this optimization can be slow in practice, and it often leads to excessively high-rank solutions.}
\end{remark}

\subsection{Maximum-Likelihood Inference of Latent Position Vectors}\label{subsec:inference-method}
Here, we present an efficient, asymptotically exact spectral algorithm for the maximum-likelihood (ML) inference problem over the logistic RDPG, subject to mild constraints. We assume that the number of dimensions is given. In practice, these parameters are often set manually, but approaches have been proposed to automatically detect the number of dimensions \cite{nc}.

The proposed maximum-likelihood inference of latent position vectors for the logistic RDPG is described in Algorithm \ref{algorithm}. In the following, we sketch the derivation of this algorithm in a series of lemmas. Proofs for these assertions are presented in Section \ref{sec:proofs}.

\begin{algorithm*}[t]
\caption{ML Inference for the logistic RDPG}
\label{algorithm}
\begin{algorithmic}
\REQUIRE Adjacency matrix $A$, number of dimensions $d$. (optional) number of clusters $k$
\STATE Form the mean-centered adjacency matrix $B:=A-1/(n(n-1))\|A\|$.
\STATE Compute $d$ eigenvectors of $B$ with largest eigenvalues: $e_1,...,e_d$.
\STATE Let $X_i=e_i e_i^T$ for $1\leq i\leq d$.
\STATE Perform logistic regression of the entries of $A$ lying above the diagonal on the corresponding entries of $X_1,....,X_d$, estimating coefficients $\ll_1^*,...\ll_d^*$ subject to the constraint that $\ll_i\geq 0\quad \forall i$.
\STATE Let $V$ be the matrix formed by concatenating $\sqrt{\ll_1^*}e_1,...,\sqrt{\ll_d^*}e_d$. Return $V$.
\STATE (optional) Perform $k$-means on $V$, and return the inferred clusters.
\end{algorithmic}
\end{algorithm*}

First, we simplify the likelihood function of Optimization \eqref{opt:ML-RDPG} using the logistic link function. Let $F(X)$ be the log-likelihood, and let the link function be $g(x)=l(x-\mu)$. Then:
\bal
F(X):=&\sum_{i,j} A_{ij}\log l(X_{ij}-\mu) + (1-A_{ij})\log(1-l(X_{ij}-\mu))\nonumber\\
=&\sum_{i,j} A_{ij}\log\ff{l(X_{ij}-\mu)}{1-l(X_{ij}-\mu)}+\log(1-l(X_{ij}-\mu))\nonumber\\
=&\sum_{i,j} A_{ij}(X_{ij}-\mu)+\log(1-l(X_{ij}-\mu)).
\end{align}
We have used that $\log(l(x)/(1-l(x)))=x$. The maximum likelihood problem takes the following form (for given $\mu$):
\begin{align}\label{opt:0}
\max_{X} \quad & Tr(AX) + \sum_{i\neq j}\log(1-\ff{1}{1+e^{-(X_{ij}-\mu)}}),\nonumber\\
&X=VV^T,\quad V\in \mathbb{R}^{n\times d}.
\end{align}
The objective function has been split into a linear term that depends on the adjacency matrix $A$ and a penalty term that does not depend on $A$. This simplification is what leads to tractable optimization, and it is the reason that the logistic link is needed; using e.g. a linear link, an optimization of this form is not obtained.

We define a penalty function $f(x)$ that keeps only the quadratic and higher-order terms in the penalty term of \eqref{opt:0}. Let
$$f(x):=-(h(x)-h(0)-h'(0)x),\quad h(x)=\log(1-l(x-\mu)).$$
Now, $h'(0)=-l(-\mu)$. Let
$$B:=A-l(-\mu){\bf 1}_{n\times n}$$
in order to re-write Optimization \eqref{opt:0} as:

\begin{align}\label{opt:sdp-exact}
\max_{X} \quad & Tr(BX) - \sum_{i\neq j}f(X_{ij}),\\
&X=VV^T,\quad V\in \mathbb{R}^{n\times d}. \nonumber
\end{align}
Note that for a centered RDPG with average density $\bar{A}$, $\mu=l\ii(\bar{A})$, and $B$ is the mean-centered adjacency matrix $A-\bar{A}{\mathbb 1}_{n\times n}$. In the next step, we convert the penalty term in the objective function into a constraint:
\begin{lemma}\label{lemma1}
Suppose that $X^*$ is the optimal solution to Optimization \eqref{opt:ML-RDPG}. Let $$h^*:=\ff{1}{n(n-1)}\sum_{i,j}f(X_{ij}^*).$$
Then $X^*$ is also the solution to the following optimization:
\begin{align}\label{opt:pen-to-bound}
\max_{X} \quad & Tr(BX)\\
&X=VV^T,\quad V\in \mathbb{R}^{n\times d} \nonumber\\
&\ff{1}{n(n-1)}\sum_{i\neq j}f(X_{ij})\leq h^*.\nonumber
\end{align}
\end{lemma}
In the following key lemma, we show that the inequality constraint of Optimization \eqref{opt:pen-to-bound} can be replaced by its second order Taylor approximation.
\begin{lemma}\label{lemma2}
For any $\epsilon>0$ and $\gamma\geq 1,$ there exists $\delta>0$ such that for any graph whose ML solution $X^*$ satisfies
\bal\label{eq:hypothesis}
h^*\leq\delta\quad\text{and}\quad\max_i X_{ii}^*\leq \ff{\gamma}{n}\sum_i X_{ii}^*,
\end{align}
the following bound is satisfied. Let $B$ be the mean centered adjacency matrix of the chosen graph. Let $s^*\in {\mathbb R}$ be the optimal value of the following optimization, obtained at $X=X^*$:
\bal\label{opt1}
\max_X \quad & Tr(BX),\\
&X=VV^T,\quad V\in \mathbb{R}^{n\times d} \nonumber\\
&\ff{1}{n(n-1)}\sum_{i\neq j} f(X_{ij})\leq h^*,\nonumber
\end{align}
Let $\tilde{s}$ be the optimal value of the following optimization:
\bal\label{opt:taylor}
\max_X \quad & Tr(BX),\\
&X=VV^T,\quad V\in \mathbb{R}^{n\times d} \nonumber\\
&\ff{a_2}{n(n-1)}\sum_{i\neq j} X_{ij}^2\leq h^*. \nonumber
\end{align}
where $a_2:=f''(-\mu)$. Then
$$\tilde{s} \geq (1-\epsilon)s^*.$$
\end{lemma}
The parameter $h^*$ is related to the average length of latent-position vectors $(X_{ii}^*)$. If these lengths approach zero, $h^*$ approaches zero, for a fixed $\gamma$. An implication of this constraint is that the logistic RDPG must be approximately centered. Thus, there is a natural choice for the parameter $\mu$ for the purpose of inference:
\bal
\hat{\mu}=-l\ii(\ff{1}{n(n-1)}\|A\|_F).
\end{align}
This estimator of $\mu$ can be viewed as the maximum-likelihood estimator specifically over the centered logistic RDPG. With this choice of $\mu$, $B$, the mean-centered adjacency matrix, can be written as $$B=A-\ff{1}{n(n-1)}\|A\|_F.$$

Note that changing the constant in the inequality constraint of Optimization \eqref{opt:taylor} only changes the scale of the solution, since the shape of the feasible set does not change. Thus, in this optimization we avoid needing to know $h^*$ {\it a priori} (as long as the conditions of Lemma \ref{lemma2} are satisfied).

Next we show that that the solution to Optimization \eqref{opt:taylor} can be recovered up to a linear transformation using spectral decomposition:
\begin{lemma}\label{lemma3}
Let $\tilde{X}$ be the optimal solution to Optimization \eqref{opt:taylor}. Let $e_1$,..., $e_d$ be the first $d$ eigenvectors of $B$, corresponding to the largest eigenvalues. Then $e_1,..,e_d$ are identical to the non-null eigenvectors of $\tilde{X}$, up to rotation.
\end{lemma}
Once the eigenvectors of $\tilde{X}$ are known, it remains only to recover the corresponding eigenvalues. Instead of recovering the eigenvalues of $\tilde{X}$, we find the eigenvalues that maximize the likelihood, given the eigenvectors of $\tilde{X}$. Let $\tilde{X}_i=e_ie_i^T$. Then, the maximum-likelihood estimate of $\lambda_1,...,\lambda_d$ conditional on $X_1,...,X_d=\tilde{X}_1,...,\tilde{X}_d$ can be written as follows:
\bal\label{opt:lambda}
 \lambda^*:=\text{argmax}_{\lambda=(\lambda_1,...,\lambda_d)} \sum_{i\neq j} \log P(A_{ij}|\tilde{X}_1,...,\tilde{X}_d, \ll,\mu).
\end{align}
\begin{lemma}\label{lemma4}
Optimization \eqref{opt:lambda} can be solved by logistic regression of the entries of $A$ on the entries of $X_1^*,...,X_d^*$, with the constraint that the coefficients are nonnegative, and with intercept $\mu$.
\end{lemma}
\begin{figure*}[t]
\hspace{.9cm}
      \includegraphics[width=.9\textwidth]{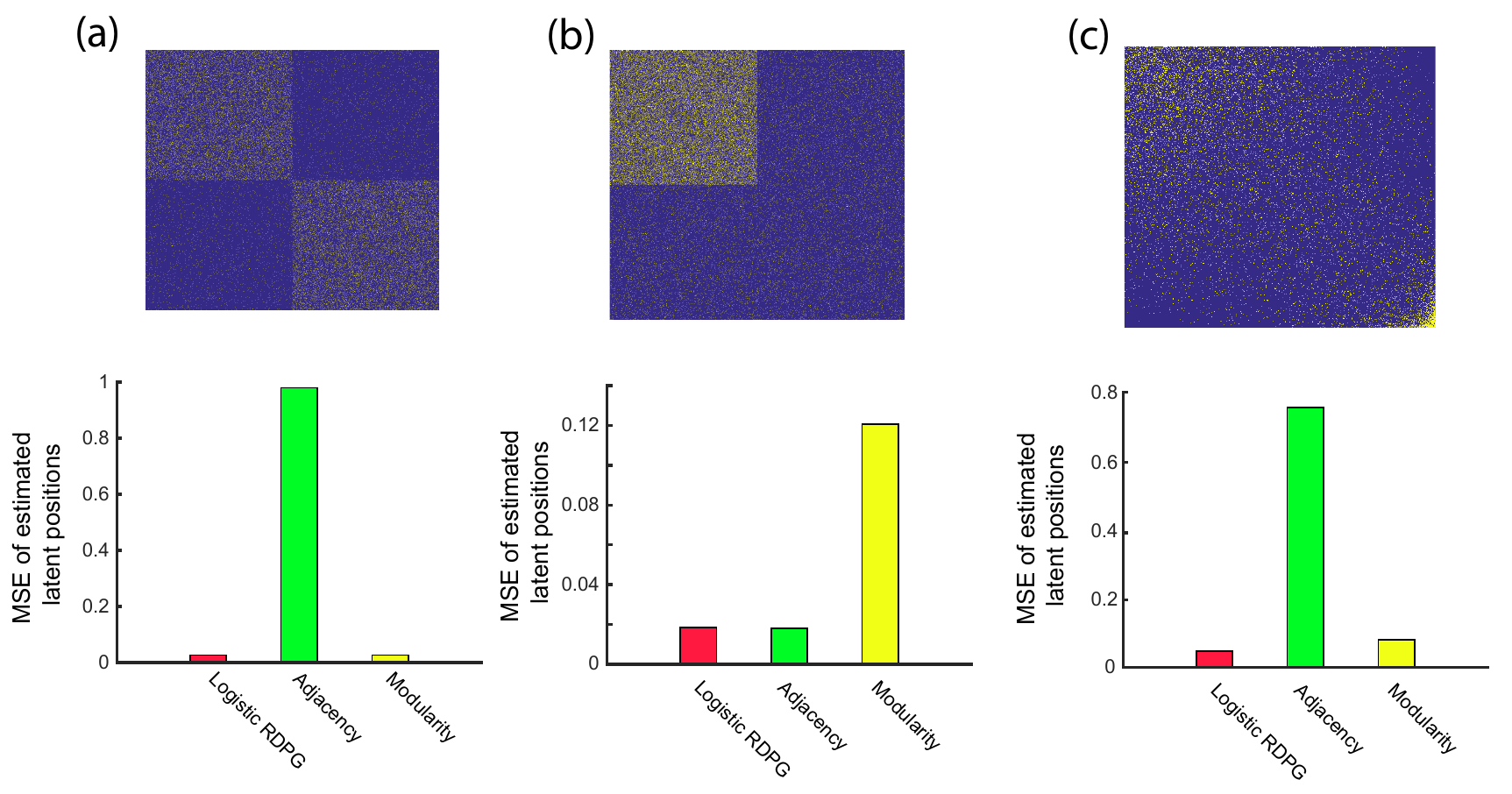}
  \caption{Normalized mean squared error (one minus squared correlation) of inferred latent positions for two SBMs (a-b) and a non-SBM logistic RDPG (c). The top eigenvectors of the adjacency matrix $A$ and the modularity matrix $M$ do not characterize the community structure in panel (a) and in panel (b), respectively. Note that in practice, a different eigenvector could be selected or multiple eigenvectors could be used. In panel (c), the top eigenvector of $A$ does not recover the latent structure. In contrast, our method successfully recovers the underlying latent position vectors in all cases.}
  \label{fig:latentpositions}
\end{figure*}

\begin{figure*}[t]
\includegraphics[width=0.8\textwidth]{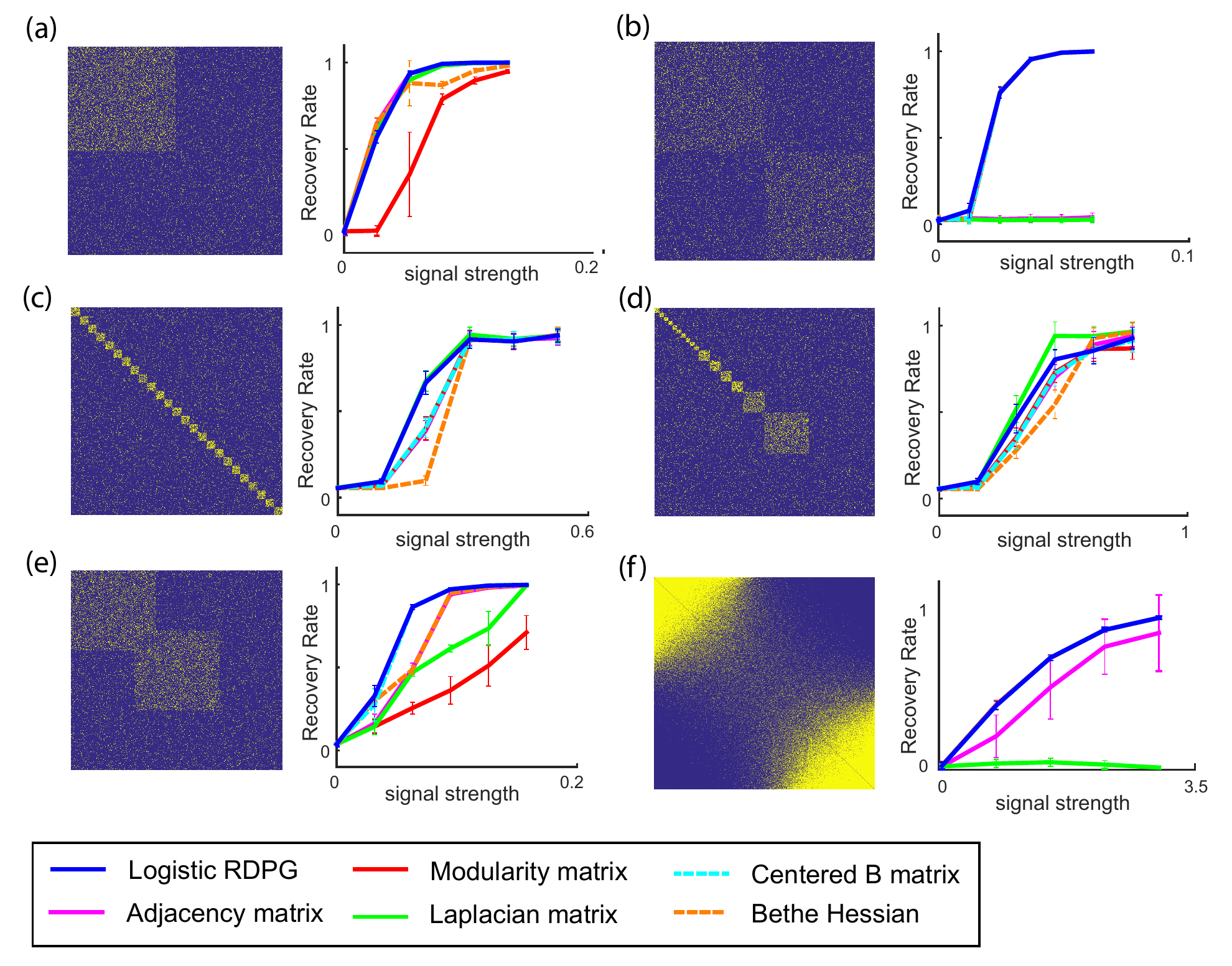}
\centering
\caption{Performance comparison of our method (logistic RDPG) against spectral clustering methods in different clustering setups. Panels (a)-(e) illustrate networks that can be characterized by SBM, while panel (f) illustrates a non-SBM network model. The scale of the $x$ axis is different in panel (f) than the rest of the panels. Our proposed method performs consistently well, while other methods exhibit sensitive and inconsistent performance in different network clustering setups. Note that in some cases, such as for the Laplacian in panel (b), performance is improved by using a different eigenvector or by using a larger number of eigenvectors.}\label{fig:simulations-a}
\end{figure*}

\begin{figure*}[t]
\includegraphics[width=.7\textwidth]{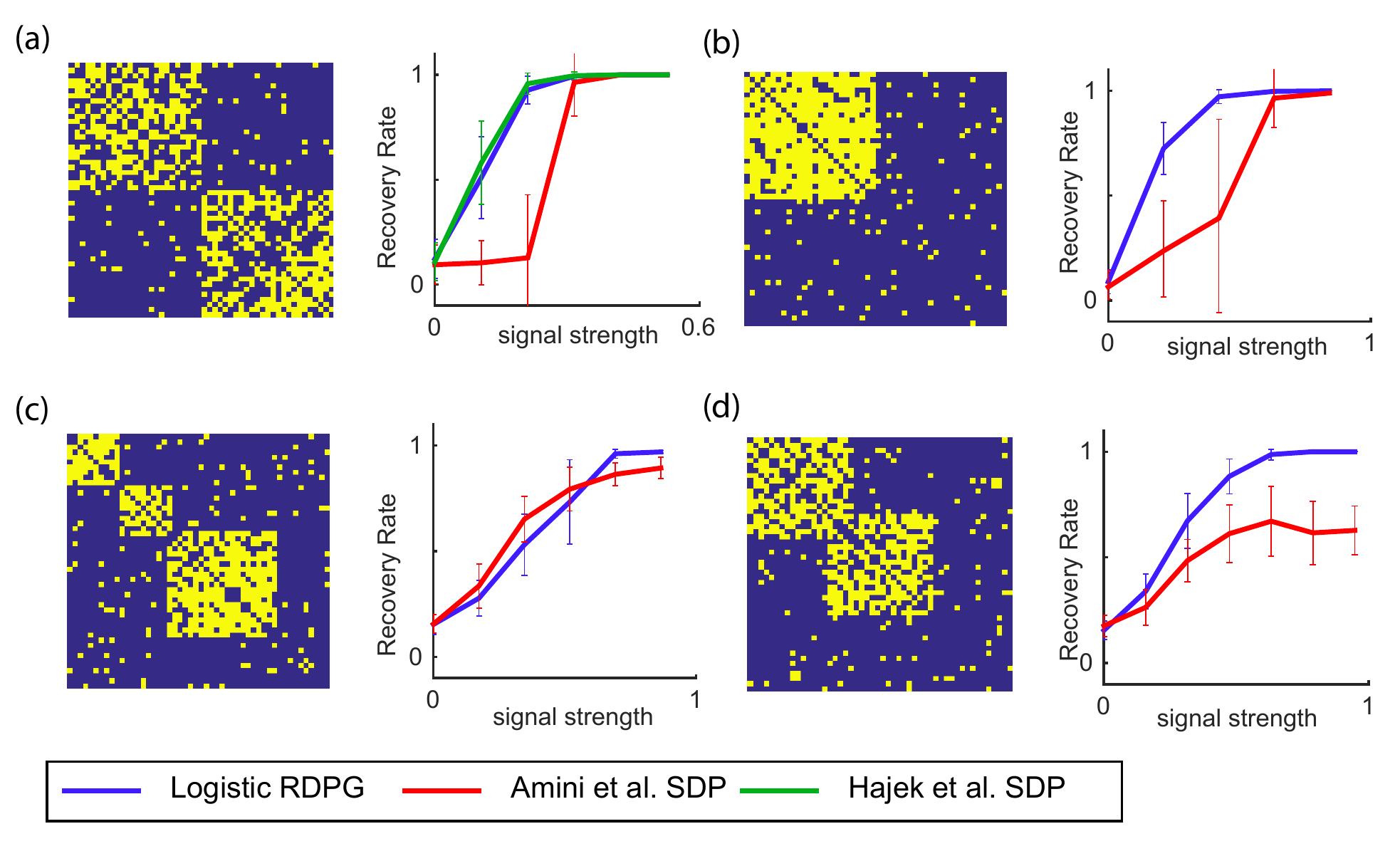}
\centering
\caption{Performance comparison of our method (logistic RDPG) against semidefinite programming-based clustering methods. Hajek et al.'s method is designed for the case of having two equally sizes partitions, thus it is not included in Panels b-d.}\label{fig:SDP-performance}
\end{figure*}

These lemmas can be used to show the asymptotic optimality of Algorithm \ref{algorithm}:
\begin{theorem}\label{thm:ml}
For all $\epsilon>0$ and $\gamma>1$, there exists $\delta>0$ that satisfies the following. For any graph with size $n$ and adjacency matrix $A$, suppose that $X^*$ is the solution to the optimization
\bal
\max_X\quad& P(A|X),\nonumber\\
&X=VV^T,\quad V\in \mathbb{R}^{n\times d}. \nonumber
\end{align}
Let $$h^*:=\ff{1}{n(n-1)}\sum_{i\neq j}f(X_{ij}^*).$$
If
\bal\label{eq:hypothesis}
h^*<\delta\quad\text{and}\quad\max_i X_{ii}^*\leq \ff{\gamma}{n}\sum_i X_{ii}^*,
\end{align}
then
$$\ff{P(A|X=X_*)}{P(A|X=X^*)}>1-\epsilon,$$
where $X_*$ is the solution obtained by Algorithm \ref{algorithm}.
\end{theorem}

Our algorithm is asymptotically exact in the sense that the likelihood ratio between our solution and the true maximum converges uniformly to one as the average latent position length shrinks. Importantly, the convergence is uniform over arbitrarily large graphs; therefore, this regime contains most interesting large network models, such as an SBM with large communities that cannot be perfectly recovered. Coupling this algorithm with a $k$-means post processing step leads to a clustering method with robust performance under different network clustering setups.

This result is stronger than the statement that an approximate objective function approximates the likelihood function at the optimum of the likelihood function. Such a result, which can be obtained for many link functions (such as an affine-linear link), is not useful because it does not follow that the optimum of the approximate function lies near the optimum of the true likelihood function. Indeed, for a linear approximation, it has no optimum since the objective function is unbounded. In order to obtain the stronger statement that the likelihood at the optimum of the approximation is large, it is necessary to use a quadratic approximation. For link functions besides the logistic link, the quadratic term in the likelihood function depends on $A$, and a spectral optimization method cannot be obtained.

The condition in Theorem \ref{thm:ml} that the lengths of optimal latent vectors are sufficiently small is not restrictive for large networks. Consider a sequence of increasingly-large SBMs with two clusters of fixed relative sizes, and a convergent sequence of admissible connectivity matrices whose  average density is fixed. There are three asymptotic regimes for the community structure: (1) in which the structure of the network is too weak to detect any clusters at all; (2) in which the communities can be partially recovered, but some mis-assignments will be made; and (3) in which the communities can be recovered perfectly. The true latent position lengths go to zero in regimes (1) and (2) as well as in part of regime (3) \cite{hajek}. Theorem \ref{thm:ml} requires maximum-likelihood latent position lengths, rather than true position lengths, to go to zero. If this is the case, and if maximum likelihood achieves the optimum thresholds for partial recovery and perfect recovery, then our method will as well.

\section{Performance Evaluation over Synthetic Networks}\label{sec:validation}
In this section, we compare the performance of our proposed method (Algorithm \ref{algorithm}) with existing methods. First, we assess the performance of our algorithm against existing methods in inference of latent position vectors of two standard SBMs depicted in Figure \ref{fig:latentpositions}. The network demonstrated in panel (a) has two dense clusters. In this case, the first eigenvector of the modularity matrix $M$ leads to a good estimation of the latent position vector while the first eigenvector of the adjacency matrix $A$ fails to characterize this vector. This is because the first eigenvector of the adjacency matrix correlates with node degrees. The modularity transformation regresses out the degree component and recovers the community structure. However, the top eigenvector of the modularity matrix fails to identify the underlying latent position vector when there is a single dense cluster in the network, and the community structure {\it is} correlated with node degrees (Figure \ref{fig:latentpositions}-b). This discrepancy highlights the sensitivity of existing heuristic inference methods in different network models (the Modularity method has not previously been considered a latent-position inference method, but we believe that its appropriate to do so). In contrast, our simple normalization allows the underlying latent position vectors to be accurately recovered in both cases. We also verified in panel (c) that our method successfully recovers latent positions for a non-SBM logistic RDPG. In this setup, the adjacency matrix's first eigenvector again correlates with node degrees, and the modularity normalization causes an improvement. We found it remarkable that such a simple normalization (mean centering) enabled such significant improvements; using more sophisticated normalizations such as the Normalized Laplacian and the Bethe Hessian, no improvements over were observed (data not shown).

Second, we assessed the ability of our method to detect communities generated from the SBM. We compared against the following existing spectral network clustering methods:
\begin{itemize}
\item Modularity (Newman, 2006). We take the first $d$ eigenvectors of the modularity matrix $M:=A-v^Tv/2|\cE|$, where $v$ is the vector of node degrees and $|\cE|$ is the number of edges in the network. We then perform $k$-means clustering on these eigenvectors.
\item Normalized Laplacian (Chung, 1997). We take second- through $(d+1)$st- last eigenvectors of $L_{sym}:=D^{-1/2}(D-A)D^{-1/2}$, where $D$ is the diagonal matrix of degrees. We then perform $k$-means clustering on these eigenvectors.
\item Bethe Hessian (Saade et al., 2014). We take the second- through $(d+1)$st- last eigenvectors of $$H(r):=(r^2-1){\bf 1}_{n\times n}-rA+D,$$ where $r^2$ is the density of the graph as defined in \cite{saade}.
\item Unnormalized spectral clustering (Sussman et al., 2012). We take the first $d$ eigenvectors of the adjacency matrix $A$, and perform $k$-means clustering on these eigenvectors.
\item Spectral clustering on the mean-centered matrix $B$. We take the first $d$ eigenvectors of the matrix $B$ and perform $k$-means on them, without a scaling step.
\end{itemize}
Note that in our evaluation we include spectral clustering on the mean-centered adjacency matrix $B$ without subsequent eigenvalue scaling of Algorithm \ref{algorithm} to demonstrate that the scaling step computed by logistic regression is essential to the performance of the proposed algorithm. When $d=1$, the methods are equivalent. We also compare the performance of our method against two SDP-based approaches, the method proposed by Hajek et al. (2015) and the SDP-1 method proposed by Amini et al. (2014). For all methods we assume that the number of clusters $k$ is given.

In our scoring metric, we distinguish between clusters and communities: For instance, in Figure \ref{fig:simulations-a}-e, there are two clusters and four communities, comprised of nodes belonging only to cluster one, nodes belonging only to cluster two, nodes belonging to both clusters, and nodes belonging to neither. The score that we use is a normalized Jaccard index, defined as:
\bal
 \ff{\max_{\sigma\in S_k}\sum_{l=1}^k \ff{|C_l\cap\hat{C}_{\sigma(l)}|}{|C_l|} - 1}{k-1}
\end{align}
where $C_l$ is the $l$-th community, $\hat{C}_l$ is the $l$-th estimated community, and $S_k$ is the group of permutations of $k$ elements. Note that one advantage of using this scoring metric is that it weighs differently-sized clusters equally (it does not place higher weights on larger communities.).

Figure \ref{fig:simulations-a} presents a comparison between our proposed method and existing spectral methods in a wide range of clustering setups. Our proposed method performs consistently well, while other methods exhibit sensitive and inconsistent performance in different network clustering setups. For instance, in the case of two large clusters (b) the second-to-last eigenvector of the Normalized Laplacian fails to correlate with the community structure; in the case of having one dense cluster (a), the Modularity normalization performs poorly; when there are many small clusters (c), the performance of the Bethe Hessian method is poor. In each case, the proposed method performs at least as well as the best alternate method, except in the case of several different-sized clusters (d), when the normalized Laplacian performs marginally better. In the case of overlapping clusters (e), our method performs significantly better than all competing methods. Spectral clustering on $B$ without the scaling step also performs well in this setup; however, its performance is worse in panels (c-d) when $d$ is larger, highlighting the importance of our logistic regression step.

The values of $k$ and $d$ for the different simulations were: $k=2,d=1$; $k=2,d=1$; $k=25,d=24$; $k=17,d=16$; $k=4,d=2$; $k=2,d=1$ for panels (a)-(f), respectively. The values of $d$ are chosen based on the number of dimensions that would be informative to the community structure, if one knew the true latent positions. All networks have 1000 nodes, with background density $0.05$.

While spectral methods are most prominent network clustering methods owing to their accuracy and efficiency, other approaches have been proposed, notably including SDP-based methods, which solve relaxations of the maximum likelihood problem over the SBM. We compare the performance of our method with the SDP-based methods proposed by Hajek et al. (2015) and Amini et al. (2014) (Figure \ref{fig:SDP-performance}). In the symmetric SBM, meaning the SBM with two equally-dense, equally-large communities, we find that our method performs almost equally well as the method of Hajek et al. (2015), which is a simple semidefinite relaxation of the likelihood in that particular case. Our method also performs better than the method of Amini et al., which solves a more complicated relaxation of the SBM maximum-likelihood problem in the more general case (Figure \ref{fig:SDP-performance}).

\section{Performance Evaluation over Real Networks}\label{sec:real}
To assess the performance of the Logistic RDPG over well-characterized real networks, we apply it to two well-known real networks. First, we consider the  Karate Club network \cite{karate}. Originally a single karate club with social ties between various members, the club split into two clubs after a dispute between the instructor and the president. The network contains 34 nodes with the average degree 4.6, including two high degree nodes corresponding to the instructor and the president. Applying our method to this network, we find that the first eigenvector separates the two true clusters perfectly (Figure \ref{fig:real}-a).

In the second experiment, we consider a network of political blogs, whose edges correspond to links between blogs \cite{blogs}. This network, which contains 1221 nodes with nonzero degrees, is sparse (the average total degree is 27.4) with a number of high degree nodes (60 nodes with degrees larger than 100). The nodes in this network have been labeled as either {\it liberal} or {\it conservative}. We apply our method to this network. Figure \ref{fig:real}-b shows inferred latent positions of nodes of this network. As it is illustrated in this figure, nodes with different labels have been separated in the latent space. Note that some nodes are placed near the origin, indicating that they cannot be clustered  confidently; this is occurred owing to their low degrees as the correlation between node degrees and distances from the origin was 0.95.

\begin{figure*}[t]
\includegraphics[width=.7\textwidth]{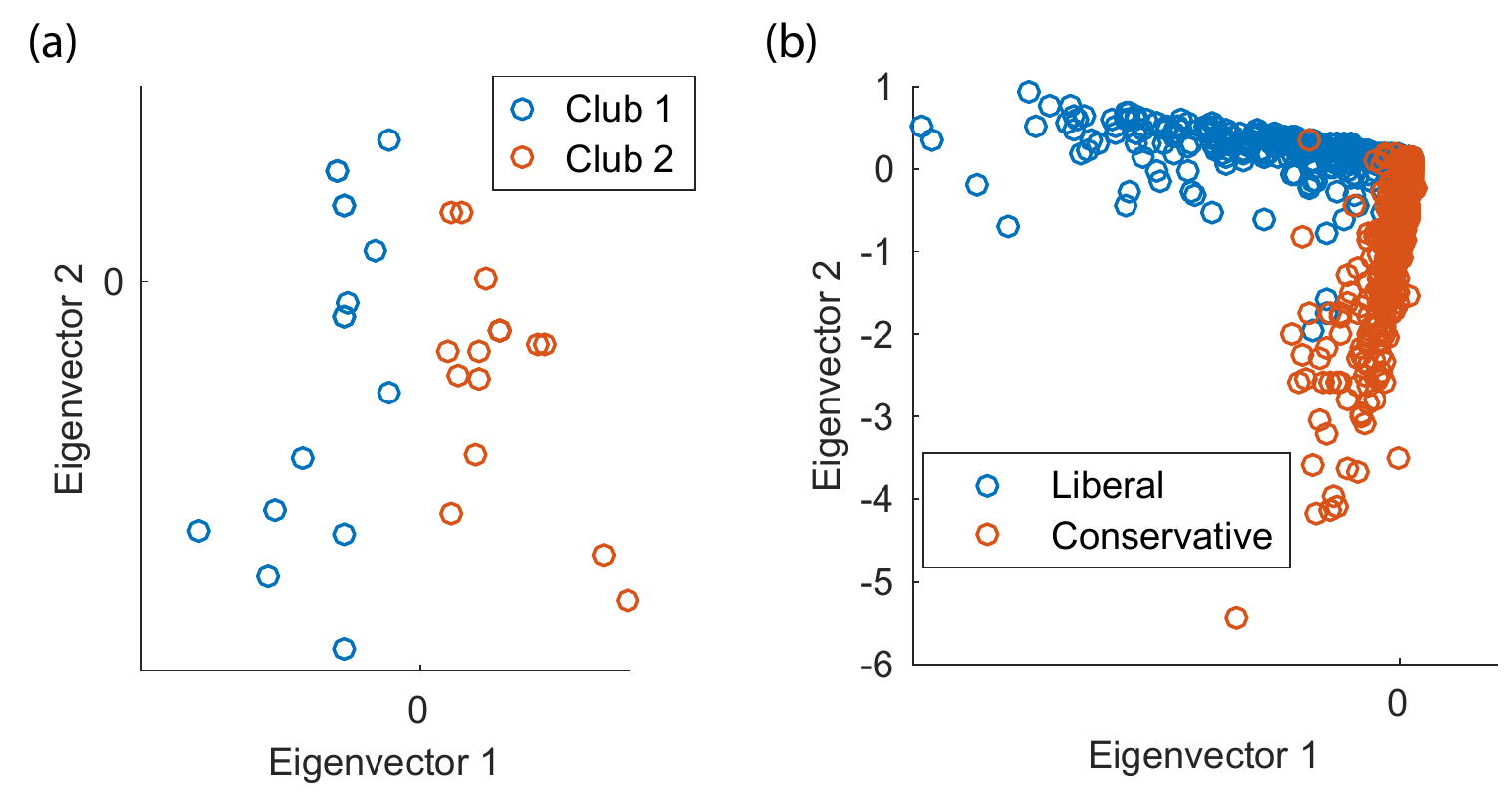}
\centering
\caption{Estimated latent positions for nodes of two real networks. (a) Karate club social network. (b) Political blogs network. }\label{fig:real}
\end{figure*}

\section{Code}
We provide code for the proposed method in the following link:
\textcolor{blue}{https://github.com/SoheilFeizi/spectral-graph-clustering}


\section{Discussion}\label{sec:discussion}
In this paper, we developed a spectral inference method over logistic Random Dot Product Graphs (RDPGs), and
we showed that the proposed method is asymptotically equivalent to the maximum likelihood latent-position inference. Previous justifications for spectral clustering have usually been either consistency results \cite{sussman, rohe} or partial-recovery results \cite{krzakala,nadakuditi}; to the best of our knowledge, our likelihood-based justification is the first of its kind for a spectral method. This type of justification is satisfying because maximum likelihood inference methods can generally be expected to have optimal asymptotic performance characteristics; for example, it is known that maximum likelihood estimators are consistent over the SBM \cite{bickel,celisse}. It remains an important future direction to characterize the asymptotic performance of the MLE over the Logistic RDPG.

We have focused in this paper on the network clustering problem; however, latent space models such as the Logistic RDPG can be viewed as a more general tool for exploring and analyzing network structures. They can be used for visualization \cite{hall,koren} and for inference of partial-membership type structures, similar to the mixed-membership stochastic blockmodel \cite{airoldi}. Our approach can also be generalized to multi-edge graphs, in which the number of edges between two nodes is binomially distributed. Such data is emerging in areas including systems biology, in the form of cell type and tissue specific networks \cite{neph}.

\section{Proofs}\label{sec:proofs}

\begin{proof}[Proof of Lemma \ref{lemma1}]
If not, then the optimal solution to Optimization \eqref{opt:pen-to-bound} would be a better solution to Optimization \eqref{opt:ML-RDPG} than $X^*$.
\end{proof}
\begin{proof}[Proof of Lemma \ref{lemma2}]
The origin is in the feasible set for both optimizations. For each optimization, the objective function value satisfies $Tr(rCX)=rTr(CX)$. Thus, the optimum is either at the origin (if there is no positive solution) or at the boundary of the feasible set. If the optimum is at the origin, we have $s^*=\tilde{s}=0$. If not, let $X$ be be any solution to $\ff{1}{n^2}\sum_{i,j} f(X_{ij})= h$. Let $r=||X||_F$, and let $r'=\sqrt{h/a_2}$. Claim: fixing $\gamma$, $r/r'\arr 1$ uniformly as $h^*\arr 0$.

\noindent
Define
$$F_{X}(a):=\ff{1}{n^2}\sum_{i,j} f(aX_{ij}/||X||_F)$$
for $a>0$. In addition, since $r'$ has been defined such that the quadratic term of $F_X(r')$ is $a_2\sum_{i,j}(\ff{r'}{r}X_{ij})^2=h^*$, we have
\bal\label{eq:move}
F_X(r')=h^*+\ff{1}{n^2}O(\sum_{i,j}(\ff{r'}{r}X_{ij})^3).
\end{align}
Moreover, the Taylor series for $f(\cdot)$ converges in a neighborhood of zero. Because of the constraint $$\max_{i,j} X_{ij}^*=\max_i X_{ii}^*\leq \ff{\gamma}{n}\sum_i X_{ii}^*,$$ we can choose $\delta$ such that every entry $X_{ij}$ falls within this neighborhood. This constraint also implies
$$\ff{1}{n^2}\sum X_{ij}^3\leq \ff{1}{n^2}\sum |X_{ij}|^3=O\left(\left(\ff{1}{n^2}||X||_F\right)^3\right).$$
Substituting this into \eqref{eq:move}, we have
\bal
F_X(r')=h^*+\ff{1}{n^2}O(r'^3).
\end{align}
Therefore, we have
\bal
\ff{|F_X(r)-F_X(r')|}{F_X(r')}=\ff{|h^*+O(r'^3)-h^*|}{h^*}=O(r').
\end{align}

\noindent
Note that $f(\cdot)$ is convex function with $f'(x)>0$ for all $x>0$ and $f'(x)<0$ for all $x<0$.
Thus $F_X$ is increasing, convex, and zero-valued at the origin: for any $a,b>0$,
\bal
 \ff{|a-b|}{b}<\ff{|F_X(a)-F_X(b)|}{F_X(b)}.
\end{align}
Thus $\ff{|r-r'|}{r'}=O(r')$ and $\ff{r}{r'}=1+O(r').$

\noindent
Let $r_s$ be the norm of the $\arg\max$ to optimization \eqref{opt1}; because the objective function is linear we have that $\tilde{s}\geq \ff{r'}{r_s}s^*$. Let $r_t$ be the distance to the intersection of the boundary of the feasible set with the ray from the origin through the $\arg\max$ to optimization \eqref{opt:taylor}; then $s^*\geq \ff{r_t}{r'}\tilde{s}$. We have shown that both ratios tend uniformly to one. This completes the proof.
\end{proof}
\begin{proof}[Proof of Lemma \ref{lemma3}]
First, suppose we have prior knowledge of eigenvalues of $X^*$. Denote its nonzero eigenvalues by $\lambda_1^*,...,\lambda_d^*$. Then we would be able to recover the optimal solution to Optimization \eqref{opt:taylor} by solving the following optimization
\bal\label{opt5}
\max_X \quad &Tr(BX) \nonumber\\
&\lambda_i=\lambda_i^*\quad 1\leq i\leq d\nonumber\\
&rank(X)=d
\end{align}

\noindent
Note that the Frobenius norm of a matrix is determined by eigenvalues of the matrix as follows:
\bal\label{norm-redundancy}
||X||_F^2=Tr(XX^T)=Tr(X^2)=\sum \ll_i^2.
\end{align}
Thus we can drop the Frobenius norm constraint in \eqref{opt:taylor}. Let $X$ be an $n\times n$ psd matrix, whose non-null eigenvectors are the columns of a matrix $E\in {\mathbb R}^{n\times d}$, and whose respective eigenvalues are $\ll_1,...,\ll_d$. Let $V:=E\ diag(\sqrt{\lambda_1},...,\sqrt{\lambda_d})$, so that $X=VV^T$. Rewrite the objective function as
$$Tr(BX)=Tr(V^TBV)=\sum_{i=1}^d \ll_i e_i^T B e_i.$$
Therefore $\tilde{X}=EE^T$ and $X^*=VV^T$.
\end{proof}
\begin{proof}[Proof of Lemma \ref{lemma4}]
The upper-triangular entries of $A$ are independent Bernoulli random variables conditional on $X$ and $\mu$, with a logistic link function. The coefficients should be nonnegative, as $X$ is constrained to be positive semidefinite.
\end{proof}
\begin{proof}[Proof of Theorem \ref{thm:ml}]
By Lemma \ref{lemma1}, we have that the solution to optimization \eqref{opt:pen-to-bound} is equal to the log-likelihood, up to addition of a constant. By Lemma \ref{lemma2}, we have that for a fixed $\gamma$, as $h^*\arr 0$, the quotient $s^*/\tilde{s}$ converges uniformly to one, where $s^*$ is the solution to \eqref{opt:pen-to-bound} and $\tilde{s}$ is the solution to optimization \eqref{opt:taylor}. The convergence is uniform over the choice of $B$ that is needed for Theorem \ref{thm:ml}. Because $s^*$ and $\tilde{s}$ do not diverge to $\pm\infty$, this also implies that $s^*-\tilde{s}$, and therefore the log-likelihood ratio, converges uniformly to zero. By Lemma \ref{lemma3}, the non-null eigenvectors of the $\arg\max$ of optimization \eqref{opt:taylor} are equivalent (up to rotation) to the first eigenvectors of $B$. Finally, by Lemma \ref{lemma4}, the eigenvalues that maximize the likelihood can be recovered using a logistic regression step. By Lemma \ref{lemma2}, the theorem would hold if we recovered the eigenvalues solving the approximate optimization \eqref{opt:taylor}. By finding the eigenvalues that exactly maximize the likelihood, we achieve a likelihood value at least as large.
\end{proof}


\end{document}